\documentclass[aos,preprint]{imsart}

\RequirePackage[OT1]{fontenc}
\RequirePackage{amsthm,amsmath,natbib}
\RequirePackage[colorlinks,citecolor=blue,urlcolor=blue]{hyperref}
\RequirePackage{hypernat}
\RequirePackage{amssymb}

\RequirePackage{epsfig}

\startlocaldefs
\numberwithin{equation}{section}
\theoremstyle{plain}

\newtheorem{lemma}{Lemma}[section]

\newtheorem{theorem}{Theorem}[section]
\newtheorem{corollary}{Corollary}[section]
\newtheorem{definition}{Definition}[section]
\newtheorem{example}{Example}[section]
\endlocaldefs

\begin{document}

\def\ci{\!\perp\!}
\def\nci{\!\not\perp\!}
\def\de{\!\sim\!}

\begin{frontmatter}
\title{Reading Dependencies from Covariance Graphs}
\runtitle{Reading Dependencies from Covariance Graphs}

\begin{aug}
\author{\fnms{Jose M.} \snm{Pe\~{n}a}\ead[label=e1]{jose.m.pena@liu.se}}

\runauthor{J. M. Pe\~{n}a}

\affiliation{ADIT, Department of Computer and Information Science\\
        Link\"oping University, SE-58183 Link\"{o}ping, Sweden\\
\printead{e1}\\
\phantom{E-mail: jose.m.pena@liu.se}}

\end{aug}

\begin{abstract}
The covariance graph (aka bi-directed graph) of a probability distribution $p$ is the undirected graph $G$ where two nodes are adjacent iff their corresponding random variables are marginally dependent in $p$.\footnote{It is worth mentioning that our definition of covariance graph is somewhat non-standard. The standard definition states that the lack of an edge between two nodes of $G$ implies that their corresponding random variables are marginally independent in $p$. This difference in the definition is important in this paper.} In this paper, we present a graphical criterion for reading dependencies from $G$, under the assumption that $p$ satisfies the graphoid properties as well as weak transitivity and composition. We prove that the graphical criterion is sound and complete in certain sense. We argue that our assumptions are not too restrictive. For instance, all the regular Gaussian probability distributions satisfy them.
\end{abstract}

\begin{keyword}
\kwd{chain graphs}
\kwd{concentration graphs}
\kwd{covariance graphs}
\end{keyword}

%\begin{keyword}[class=AMS]
%\kwd{62H05, 60E05, 68T30}
%\end{keyword}

\end{frontmatter}

\section{Introduction}\label{sec:intro}

The covariance graph (aka bi-directed graph) of a probability distribution $p$ is the undirected graph $G$ where two nodes are adjacent iff their corresponding random variables are marginally dependent in $p$. Covariance graphs were introduced in \citep{CoxandWermuth1993} to represent independence models. Since then, they have received considerable attention. See, for instance, \citep{BanerjeeandRichardson2003,Chaudhurietal.2007,CoxandWermuth1996,DrtonandRichardson2003,DrtonandRichardson2008,Kauermann1996,Lupparellietal.2009,MaloucheandRajaratnam2009,PearlandWermuth1994,Richardson2003,Wermuth1995,WermuthandCox1998,Wermuthetal.2006,Wermuth2011,Wermuth2012}. The works \citep{BanerjeeandRichardson2003,Kauermann1996} are particularly important for the interpretation of covariance graphs in terms of independencies. Specifically, these works introduce a graphical criterion for reading independencies from the covariance graph $G$ of a probability distribution $p$, under the assumption that $p$ satisfies the graphoid properties and composition. In this paper, we show that $G$ can also be used to read dependencies holding in $p$. Specifically, we present a graphical criterion for reading dependencies from $G$ under the assumption that $p$ satisfies the graphoid properties, weak transitivity and composition. We also prove that our graphical criterion is sound and complete. Here, complete means that it is able to read all the dependencies in $p$ that can be derived by applying the graphoid properties, weak transitivity and composition to the dependencies used in the construction of $G$ and the independencies obtained from $G$. We also show that there exist important families of probability distributions that satisfy the graphoid properties, weak transitivity and composition. These include, for instance, all the regular Gaussian probability distributions.

Note that this paper would be unnecessary if $p$ satifies all and only the independencies that can be read from $G$ via the graphical criterion in \citep{BanerjeeandRichardson2003,Kauermann1996}, i.e. $p$ is faithful to $G$. We will see that one cannot safely assume faithfulness in general. Therefore, one is only entitled to assume that $p$ satifies all (but not necessarily only) the independencies that can be read from $G$ via the graphical criterion in \citep{BanerjeeandRichardson2003,Kauermann1996}, i.e. $p$ is Markov wrt $G$. This is actually the reason of being of this paper.

Two previous works that somehow address the problem of reading dependencies off covariance graphs are \citep{Wermuth1995,WermuthandCox1998}. These works propose to determine whether two random variables $U_A$ and $U_B$ are dependent given some other random variables $U_Z$ by, first, constructing the covariance graph of the conditional probability distribution given $U_Z$ of any set of random variables that includes $U_A$ and $U_B$ and, then, checking if the nodes corresponding to $U_A$ and $U_B$ are adjacent in the covariance graph constructed. Therefore, these works construct multiple covariance graphs, one for each conditional probability distribution of interest, from which only the dependencies used in their construction are read. The work presented in this paper is radically different: We only construct the covariance graph of the probability distribution at hand and read from it many more dependencies than those used in its construction. While this is the first work where a sound and complete graphical criterion for reading dependencies off covariance graphs is developed, it is worth mentioning that there already exist sound and complete graphical criteria for reading dependencies off other graphical models. For instance, there exists a sound and complete graphical criterion for reading dependencies off the concentration graph (aka minimal undirected independence map or Markov network) of a probability distribution that satisfies the graphoid properties \citep{Bouckaert1995}, or the graphoid properties and weak transitivity \citep{Pennaetal.2009}. As a matter of fact, the graphical criterion that we present in this paper is dual to the one in \citep{Pennaetal.2009}. There also exists a sound and complete graphical criterion for reading dependencies off the Bayesian network (aka minimal directed independence map) of a probability distribution that satisfies the graphoid properties \citep{Bouckaert1995}, the graphoid properties and weak transitivity \citep{Penna2010}, or the graphoid properties and weak transitivity and composition \citep{Penna2007}. In the last two references, the Bayesian networks are restricted to be polytrees. Note that \citep{Bouckaert1995,Penna2007,Pennaetal.2009,Penna2010} address a related but not more general problem than the one in this paper, since neither concentration graphs nor Bayesian networks include covariance graphs. Related more general problems than the one studied in this paper have been recently addressed, though. For instance, a method to read dependencies from multivariate regression graphs, which include covariance graphs, is proposed in \citep{Wermuth2012}. The author also presents necessary and sufficient conditions for the method to be sound. These conditions are the same as the ones considered in this paper, namely the graphoid properties plus weak transitivity and composition. Unlike in this paper, no proof of completeness of the method proposed appears in \citep{Wermuth2012}. Another related more general work is \citep{Wermuth2011}, where the author shows how summary graphs, which include covariance graphs, can help to detect which dependencies remain undistorted and which do not after marginalization and/or conditioning in a probability distribution generated over a so-called parent graph. It should be pointed out that that the probability distribution is generated over a parent graph implies that it satisfies the same conditions as the ones considered in this paper \citep[Proposition 3]{Wermuth2011}. Again, unlike in this paper, the completeness question is not addressed in \citep{Wermuth2011}. Finally, it should be noted that \citep{Wermuth2011,Wermuth2012} make use of the graphical criterion presented in \citep{SadeghiandLauritzen2011} for reading independencies from loopless mixed graphs, which include multivariate regression graphs, summary graphs and parent graphs. This criterion is sound and complete in certain sense, given that the graphoid properties and composition hold \citep[Theorem 3]{SadeghiandLauritzen2011}. These conditions are, in fact, not only sufficient but necessary too \citep[Section 6.3]{SadeghiandLauritzen2011}.

We think that the work presented in this paper can be of great interest for the artificial intelligence community. Graphs are one of the most commonly used metaphors for representing knowledge because they appeal to human intuition \citep{Pearl1988}. Furthermore, graphs are parsimonious models because they trade off accuracy for simplicity. Consider, for instance, representing the independence model induced by a probability distribution as a graph. Though this graph is typically less accurate than the probability distribution (the graph may not represent all the (in)dependencies and those that are represented are not quantified), it also requires less space to be stored and less time to be communicated than the probability distribution, which may be desirable features in some applications. Thus, it seems sensible developing tools for reasoning with graphs. Our graphical criterion is one such a tool: As the graphical criterion in \citep{BanerjeeandRichardson2003,Kauermann1996} makes the discovery of independencies amenable to human reasoning by enabling to read independencies off a covariance graph $G$ without numerical calculation, so does our graphical criterion with respect to the discovery of dependencies. There are fields where discovering dependencies is more important than discovering independencies \citep{Wermuth1995,WermuthandCox1998}. It is in these fields where we believe that our graphical criterion has greater potential. In bioinformatics, for instance, the nodes of $G$ may represent (the expression levels of) some genes under study. Bioinformaticians are typically more interested in discovering gene dependencies than independencies, because the former provide contexts in which the expression level of some genes is informative about that of some other genes, which may lead to hypothesize dependencies, functional relations, causal relations, the effects of manipulation experiments, etc. See, for instance, \citep{ButteandKohane2000} for an application of covariance graphs to bioinformatics under the name of relevance networks.

The rest of the paper is organized as follows. We start by reviewing some concepts in Section \ref{sec:preliminaries}. We show in Section \ref{sec:wtc} that assuming the graphoid properties, weak transitivity and composition is not too restrictive. We prove in Section \ref{sec:indep} that the existing graphical criterion for reading independencies from covariance graphs is complete in certain sense. This result, in addition to being important in its own, is important for reading as many dependencies as possible from covariance graphs. We introduce in Section \ref{sec:dep} our graphical criterion for reading dependencies from covariance graphs and prove that it is sound and complete in certain sense. Finally, we close with some discussion in Section \ref{sec:discussion}.

\section{Preliminaries}\label{sec:preliminaries}

In this section, we introduce some concepts and results that are used later in this paper. We first recall some results from graphical models. See, for instance, \citep{BanerjeeandRichardson2003,Kauermann1996,Lauritzen1996,Studeny2005} for further information. Let $V=\{1, \ldots, N\}$ be a finite set of size $N$. The elements of $V$ are not distinguished from singletons and the union of the sets $I_1, \ldots, I_n \subseteq V$ is written as the juxtaposition $I_1 \ldots I_n$. We assume throughout the paper that the union of sets precedes the set difference when evaluating an expression. Unless otherwise stated, all the graphs in this paper are defined over $V$. If a graph $G$ contains an undirected (respectively directed) edge between two nodes $v_{1}$ and $v_{2}$, then we say that $v_{1} - v_{2}$ (respectively $v_{1} \rightarrow v_{2}$) is in $G$. If $v_{1} \rightarrow v_{2}$ is in $G$ then $v_{1}$ is called a parent of $v_{2}$ in $G$. Let $Pa_G(I)$ denote the set of parents in $G$ of the nodes in $I \subseteq V$. A route from a node $v_{1}$ to a node $v_{n}$, denoted $v_{1}:v_{n}$, in a graph $G$ is a sequence of nodes $v_{1}, \ldots, v_{n}$ such that there exists an edge in $G$ between $v_{i}$ and $v_{i+1}$ for all $1 \leq i < n$. A path is a route $v_{1}:v_{n}$ in which the nodes $v_{1}, \ldots, v_{n}$ are distinct. A route $v_{1}:v_{n}$ is called undirected if $v_{i} - v_{i+1}$ is in $G$ for all $1 \leq i < n$. A node $v_{1}$ is an ancestor of a node $v_{n}$ in $G$ if there is a route $v_{1}:v_{n}$ in $G$ such that $v_{i} - v_{i+1}$ or $v_{i} \rightarrow v_{i+1}$ is in $G$ for all $1 \leq i < n$.\footnote{Note that our definition of ancestor follows \citep{Lauritzen1996} and differs from others that exist in the literature, e.g. \citep{RichardsonandSpirtes2002}.} Let $An_G(I)$ denote the set of ancestors in $G$ of the nodes in $I \subseteq V$. A node $v_{n}$ is a descendant of a node $v_{1}$ in $G$ if there is a route $v_{1}:v_{n}$ in $G$ such that $v_{i} - v_{i+1}$ or $v_{i} \rightarrow v_{i+1}$ is in $G$ for all $1 \leq i < n$ and $v_{i} \rightarrow v_{i+1}$ is in $G$ for some $1 \leq i < n$. A chain graph (CG) is a graph (possibly) containing both undirected and directed edges and such that no node is a descendant of itself. An undirected graph (UG) is a CG containing only undirected edges. A directed and acyclic graph (DAG) is a CG containing only directed edges. A set of nodes of a CG is connected if there exists an undirected route in the CG between every pair of nodes in the set. A connectivity component of a CG is a connected set that is maximal with respect to set inclusion. The moral graph of a CG $G$, denoted $G^m$, is the UG where two nodes are adjacent iff they are adjacent in $G$ or they are both in $Pa_G(B_i)$ for some connectivity component $B_i$ of $G$. The subgraph of a CG $G$ induced by $I \subseteq V$, denoted $G_I$, is the graph over $I$ where two nodes are connected by a (un)directed edge if that edge is in $G$. Let $X$, $Y$ and $Z$ denote three disjoint subsets of $V$. We say that $X$ is separated from $Y$ given $Z$ in a CG $G$ if every path in $(G_{An_G(XYZ)})^m$ from a node in $X$ to a node in $Y$ has some node in $Z$. We denote such a separation statement by $sep_G(X, Y | Z)$.

Let $U=(U_i)_{i \in V}$ denote a vector of random variables and $U_I$ $(I \subseteq V)$ its subvector $(U_i)_{i \in I}$. We use upper-case letters to denote random variables and the same letters in lower-case to denote their states. Unless otherwise stated, all the probability distributions in this paper are defined over $U$. Let $X$, $Y$, $Z$ and $W$ denote four disjoint subsets of $V$. We represent by $X \ci_p Y | Z$ that $U_X$ is independent of $U_Y$ given $U_Z$ in a probability distribution $p$. We represent by $X \nci_p Y | Z$ that $X \ci_p Y | Z$ does not hold. A probability distribution $p$ is a graphoid if it satisfies the following properties: Symmetry $X \ci_p Y | Z \Rightarrow Y \ci_p X | Z$, decomposition $X \ci_p Y W | Z \Rightarrow X \ci_p Y | Z$, weak union $X \ci_p Y W | Z \Rightarrow X \ci_p Y | Z W$, contraction $X \ci_p Y | Z W \land X \ci_p W | Z \Rightarrow X \ci_p Y W | Z$, and intersection $X \ci_p Y | Z W \land X \ci_p W | Z Y \Rightarrow X \ci_p Y W | Z$. We say that a graphoid $p$ is a WTC graphoid if it satisfies the following two additional properties: Weak transitivity $X \ci_p Y | Z \land X \ci_p Y | Z K \Rightarrow$ \mbox{$ X \ci_p K | Z \lor K \ci_p Y | Z$} with $K \in V \setminus X Y Z$, and composition $X \ci_p Y | Z \land$ \mbox{$ X \ci_p W | Z \Rightarrow X \ci_p Y W | Z$}.

Let $X$, $Y$ and $Z$ denote three disjoint subsets of $V$. We denote by \mbox{$X \ci_G Y | Z$} that a CG $G$ represents that $U_X$ is independent of $U_Y$ given $U_Z$. We denote by $X \nci_G Y | Z$ that $X \ci_G Y | Z$ does not hold. In this paper, we are interested in the classic Lauritzen-Wermuth-Frydenberg interpretation of CGs as independence models, which is based on the following graphical criterion.

\begin{definition}\label{def:inde}
Given a CG $G$, $X \ci_G Y | Z$ if $sep_G(X, Y | Z)$.
\end{definition}

However, in this paper we are also interested in the dual interpretation of UGs as independence models that builds on the following graphical criterion.

\begin{definition}\label{def:inde2}
Given an UG $G$, $X \ci_G Y | Z$ if $sep_G(X, Y | V \setminus X Y Z)$.
\end{definition}

The following rephrasing of the graphical criterion in Definition \ref{def:inde2} may be easier to recall: $X \ci_G Y | Z$ if every path in $G$ from a node in $X$ to a node in $Y$ has some node outside $X Y Z$. When an UG is interpreted according to the graphical criterion in Definition \ref{def:inde} we call it a concentration graph, and when it is interpreted according to the graphical criterion in Definition \ref{def:inde2} we call it a covariance graph. A probability distribution $p$ is Markov wrt a CG, concentration graph or covariance graph $G$ when $X \ci_p Y | Z$ if $X \ci_G Y | Z$ for all $X$, $Y$ and $Z$ disjoint subsets of $V$. A probability distribution $p$ is faithful to a CG, concentration graph or covariance graph $G$ when $X \ci_p Y | Z$ iff $X \ci_G Y | Z$ for all $X$, $Y$ and $Z$ disjoint subsets of $V$. The concentration graph (aka minimal undirected independence map or Markov network) of a probability distribution $p$ is the UG $G$ where two nodes $A$ and $B$ are adjacent iff $A \nci_p B | V \setminus A B$. The covariance graph (aka bi-directed graph) of a probability distribution $p$ is the UG $G$ where two nodes $A$ and $B$ are adjacent iff $A \nci_p B$. A WTC graphoid $p$ is Markov wrt both its covariance graph $G$ and its concentration graph $H$. However, neither $X \nci_G Y | Z$ nor $X \nci_H Y | Z$ implies $X \nci_p Y | Z$, unless $p$ is faithful to $G$ or $H$. This is actually the reason of being of this paper.

\section{WTC Graphoids}\label{sec:wtc}

This paper is devoted to the study of WTC graphoids. We show in this section that WTC graphoids are worth studying because they include important families of probability distributions. For instance, any regular Gaussian probability distribution is a WTC graphoid \citep[Sections 2.2.2, 2.3.5 and 2.3.6]{Studeny2005}. The following theorem introduces another interesting family of WTC graphoids.

\begin{theorem}\label{the:fcg}
Let $G$ be a CG. Any probability distribution $p$ that is faithful to $G$ is a WTC graphoid.
\end{theorem}

\begin{proof}
Let $q$ be any regular Gaussian probability distribution that is faithful to $G$. Such probability distributions exist due to \citep[Theorems 1 and 2]{Penna2011}. Since $p$ and $q$ are faithful to $G$, $X \ci_p Y | Z$ iff $X \ci_G Y | Z$ iff $X \ci_q Y | Z$ for all $X$, $Y$ and $Z$ disjoint subsets of $V$. Therefore, $p$ is a WTC graphoid because $q$ is a WTC graphoid.
\end{proof}

The previous theorem is meaningful only if we prove that, for any CG, there exist probability distributions that are faithful to it. We do so in the following theorem.

\begin{theorem}\label{the:fcg2}
Let $G$ be a CG. If each random variable in $U$ has a finite prescribed sample space with at least two possible states, then there exists a discrete probability distribution with the prescribed sample spaces for the random variables in $U$ that is faithful to $G$. On the other hand, if the sample space of each random variable in $U$ is $\mathbb{R}$, then there exist a regular Gaussian probability distribution that is faithful to $G$ and a continuous but non-Gaussian probability distribution that is faithful to $G$.
\end{theorem}

\begin{proof}
The first and second statements in the theorem are proven in \citep[Theorems 3 and 5]{Penna2009} and \citep[Theorems 1 and 2]{Penna2011}, respectively. The third statement in the theorem can easily be proven by using copulas \citep{Nelsen2006} as follows. Let $p$ denote any regular Gaussian probability distribution that is faithful to $G$. Derive the Gaussian copula for $p$. The copula represents the independence model of $p$ stripped from its univariate marginals. Therefore, the copula together with a set of arbitrary univariate marginals can be used to generate a multivariate probability distribution whose independence model is the one dictated by the copula and whose univariate marginals are the given ones. The desired result is achieved if the arbitrary marginals are chosen so that they are continuous but non-Gaussian. See \citep{Nelsen2006} for more details.
\end{proof}

It is worth mentioning that the results in \citep[Theorems 3 and 5]{Penna2009} and \citep[Theorems 1 and 2]{Penna2011} are actually stronger than the first and second statements in the previous theorem. Specifically, the results reported there are that, in certain measure-theoretic sense, almost all the discrete probability distributions and regular Gaussian probability distributions that are Markov wrt a CG are faithful to it. Finally, note that the marginals and conditionals of a regular Gaussian probability distribution are regular Gaussian probability distributions and, thus, WTC graphoids. In fact, this property can be generalized to all the WTC graphoids. The following theorem, originally reported in \citep[Theorem 5]{Pennaetal.2006}, formalizes this result.

\begin{theorem}\label{the:ChickeringandMeek2002}
Let $p$ be a WTC graphoid and let $I \subseteq V$. Then, $p(U_{V \setminus I})$ is a WTC graphoid. If $p(U_{V \setminus I} | U_I = u_I)$ has the same (in)dependencies for all $u_I$, then $p(U_{V \setminus I} | U_I = u_I)$ for any $u_I$ is a WTC graphoid.
\end{theorem}

It is worth noting that many members of the families of WTC graphoids that we have presented in this section are Markov wrt their covariance graphs but not faithful to them. Hence, the need to develop a graphical criterion for reading dependencies from the covariance graph of a WTC graphoid. For example, consider any discrete, regular Gaussian, or continuous but non-Gaussian probability distribution $p$ that is faithful to a CG with $\{A \rightarrow B, B \rightarrow C\}$ as induced subgraph. Then, the covariance graph $G$ of $p$ has $\{A - B, A - C, B - C\}$ as induced subgraph and, thus, $p$ is not faithful to $G$ since $A \nci_G C | Z$ but $A \ci_p C | Z$ for some $Z \subseteq V$. This example is based on \citep{DrtonandRichardson2003,PearlandWermuth1994}. The interested reader is referred to these works for a characterization of the independence models that can be represented exactly by DAGs but not by covariance graphs.

\section{Reading Independencies}\label{sec:indep}

The graphical criterion in Definition \ref{def:inde2} is sound for reading independencies from the covariance graph $G$ of a WTC graphoid $p$, that is, it only identifies independencies in $p$ \citep[Proposition 2.2]{BanerjeeandRichardson2003,Kauermann1996}. In this section, we show that this graphical criterion is complete in the sense that it identifies all the independencies in $p$ that can be identified by studying $G$ alone. This completeness result, in addition to being important in its own, is crucial for reading as many dependencies as possible from $G$, as we will see in the next section. In order to prove the referred completeness result, it suffices to prove that there exist WTC graphoids that are faithful to $G$, because $G$ is their covariance graph and they only have the independencies that the graphical criterion in Definition \ref{def:inde2} identifies from $G$. Therefore, we cannot derive more independencies from $G$ alone than those identified by this graphical criterion, because $p$ may be one of the WTC graphoids that are faithful to $G$. The following two theorems prove the desired result.

\begin{theorem}\label{the:fcvg}
Let $G$ be a covariance graph. If each random variable in $U$ has a finite prescribed sample space with at least two possible states, then there exists a discrete probability distribution with the prescribed sample spaces for the random variables in $U$ that is faithful to $G$. On the other hand, if the sample space of each random variable in $U$ is $\mathbb{R}$, then there exist a regular Gaussian probability distribution that is faithful to $G$ and a continuous but non-Gaussian probability distribution that is faithful to $G$.
\end{theorem}

\begin{proof}
We start by proving the first statement in the theorem. First, we create a DAG from $G$ as follows: Replace each edge $A - B$ in $G$ with $A \leftarrow V_{A,B}' \rightarrow B$ where $V_{A,B}'$ is a newly created node. Call the resulting DAG $H$ and let $V'$ denote all the newly created nodes. It is easy to see that $X \ci_H Y | Z$ iff $X \ci_G Y | Z$ for all $X$, $Y$ and $Z$ disjoint subsets of $V$.

Let $U'=(U_i')_{i \in V'}$ denote a vector of random variables such that each of them has any finite sample space with at least two possible states. Let $p(U, U')$ denote any discrete probability distribution that is faithful to $H$. Such probability distributions exist by \citep[Theorem 7]{Meek1995}. Note that, for any $X$, $Y$ and $Z$ disjoint subsets of $V$, $X \ci_{p(U)} Y | Z$ iff $X \ci_{p(U, U')} Y | Z$ iff $X \ci_H Y | Z$ iff $X \ci_G Y | Z$. Consequently, $p(U)$ is faithful to $G$.

The second statement in the theorem can be proven in much the same way as the first if $p(U, U')$ denotes now any regular Gaussian probability distribution that is faithful to $H$. Such probability distributions exist by \citep[Theorem 3.2]{Spirtesetal.1993}. Note that $p(U)$ is regular Gaussian.

Finally, the third statement in the theorem can be proven by using copulas as we did in the proof of Theorem \ref{the:fcg2}.
\end{proof}

An alternative proof of the second statement in the theorem above follows from \citep[Theorem 7.5]{RichardsonandSpirtes2002}. Specifically, \citep{RichardsonandSpirtes2002} introduces a new class of graphical models called ancestral graphs, whose edges can be undirected, directed or bi-directed ($\leftrightarrow$). Covariance graphs are equivalent to ancestral graphs with only bi-directed edges. \citep[Theorem 7.5]{RichardsonandSpirtes2002} proves that, for any ancestral graph, there is a regular Gaussian probability distribution that is faithful to it. In Appendix A, we strengthen the second statement in the theorem above by proving that, in certain measure-theoretic sense, almost all the regular Gaussian probability distributions that are Markov wrt a covariance graph are faithful to it. Although this result is not used in this paper, we consider it to be interesting in its own and, thus, we decide to report on it.

\begin{theorem}\label{the:fcvg2}
Let $G$ be a covariance graph. Any probability distribution $p$ that is faithful to $G$ is a WTC graphoid.
\end{theorem}

\begin{proof}
Let $q$ be any regular Gaussian probability distribution that is faithful to $G$. Such probability distributions exist due to Theorem \ref{the:fcvg}. Since $p$ and $q$ are faithful to $G$, $X \ci_p Y | Z$ iff $X \ci_G Y | Z$ iff $X \ci_q Y | Z$ for all $X$, $Y$ and $Z$ disjoint subsets of $V$. Therefore, $p$ is a WTC graphoid because $q$ is a WTC graphoid.
\end{proof}

As explained at the beginning of this section, the previous two theorems imply that the graphical criterion in Definition \ref{def:inde2} is complete for reading independencies from the covariance graph $G$ of a WTC graphoid $p$, in the sense that it identifies all the independencies in $p$ that can be identified by studying $G$ alone. An equivalent formulation of this result is that the graphical criterion is complete in the sense that it identifies all the independencies that are shared by all the WTC graphoids whose covariance graph is $G$. Finally, it is worth mentioning that the graphical criterion in Definition \ref{def:inde2} is not complete in the more stringent sense of being able to identify all the independencies in $p$. Actually, no sound graphical criterion for reading independencies from $G$ is complete in this latter sense. An example illustrating this follows.

\begin{example}\label{exa:example}
Let $p$ and $p'$ denote two WTC graphoids that are faithful to the CGs in the left and center of Table \ref{tab:example}, respectively. Such WTC graphoids exist by \citep[Theorems 1 and 2]{Penna2011}. Note that $A \ci_p C | B$ whereas \mbox{$A \nci_{p'} C | B$}. Let $G$ and $H$ denote the covariance and concentration graphs of $p$, respectively. Likewise, let $G'$ and $H'$ denote the covariance and concentration graphs of $p'$, respectively. Note that $G$, $H$, $G'$ and $H'$ are all the complete graph over $\{A, B, C, D\}$. Now, let us assume that we are dealing with $p$. Then, no sound graphical criterion entails $A \ci_p C | B$ from $G$ because this independence does not hold in $p'$, and it is impossible to know whether we are dealing with $p$ or $p'$ on the sole basis of $G$.
\end{example}

\begin{table}[t]
\centering
\begin{tabular}{c}
\hline
\\
\includegraphics[scale=0.45]{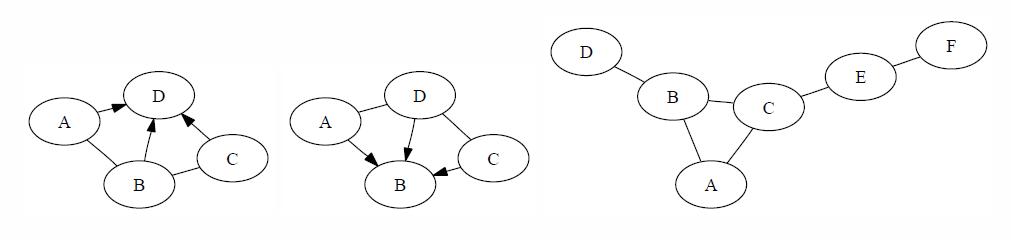}
\\
\hline
\end{tabular}
\caption{CGs in Example \ref{exa:example}, and covariance graph in Examples \ref{exa:example2} and \ref{exa:example3}.}\label{tab:example}
\end{table}

\section{Reading Dependencies}\label{sec:dep}

In this section, we present the main contribution of this paper: We introduce a graphical criterion for reading dependencies from the covariance graph of a WTC graphoid and prove that it is sound and complete in certain sense. If $G$ is the covariance graph of a WTC graphoid $p$ then we know, by definition of $G$, that $A \nci_p B$ for all the edges $A - B$ in $G$. We call these dependencies the dependence base of $p$. Further dependencies in $p$ can be derived from the dependence base via the WTC graphoid properties. For this purpose, we rephrase the WTC graphoid properties in their contrapositive form as follows. Symmetry $Y \nci_p X | Z \Rightarrow X \nci_p Y | Z$. Decomposition $X \nci_p Y | Z \Rightarrow X \nci_p Y W | Z$. Weak union $X \nci_p Y | Z W \Rightarrow X \nci_p Y W | Z$. Contraction $X \nci_p Y W | Z \Rightarrow$ \mbox{$ X \nci_p Y | Z W \lor X \nci_p W | Z$} is problematic for deriving new dependencies because it contains a disjunction in the consequent and, thus, we split it into two properties: Contraction1 $X \nci_p Y W | Z \land X \ci_p Y | Z W \Rightarrow X \nci_p W | Z$, and contraction2 $X \nci_p Y W | Z \land X \ci_p W | Z \Rightarrow X \nci_p Y | Z W$. Likewise, intersection gives rise to intersection1 $X \nci_p Y W | Z \land X \ci_p Y | Z W \Rightarrow X \nci_p W | Z Y$, and intersection2 $X \nci_p Y W | Z \land X \ci_p W | Z Y \Rightarrow X \nci_p Y | Z W$. Note that intersection1 and intersection2 are equivalent and, thus, we refer to them simply as intersection. Similarly, weak transitivity gives rise to weak transitivity1 $X \nci_p K | Z \land K \nci_p Y | Z \land X \ci_p Y | Z \Rightarrow X \nci_p Y | Z K$, and weak transitivity2 $X \nci_p K | Z \land K \nci_p Y | Z \land X \ci_p Y | Z K \Rightarrow X \nci_p Y | Z$. Finally, composition $X \nci_p Y W | Z \Rightarrow X \nci_p Y | Z \lor X \nci_p W | Z$ gives rise to composition1 $X \nci_p Y W | Z \land X \ci_p Y | Z \Rightarrow X \nci_p W | Z$, and composition2 $X \nci_p Y W | Z \land X \ci_p W | Z \Rightarrow X \nci_p Y | Z$. Since composition1 and composition2 are equivalent, we refer to them simply as composition. The independence in the antecedent of any of the properties above holds if it can be read off $G$ via the graphical criterion in Definition \ref{def:inde2}. This is the best solution we can hope for because, as discussed in the previous section, this graphical criterion is sound and complete for WTC graphoids. Moreover, this solution does not require more information than what it is available, namely $G$ or equivalently the dependence base of $p$. We define the WTC graphoid closure of the dependence base of $p$ as the set of dependencies that are in the dependence base of $p$ plus those that can be derived from it by applying the nine properties above.

Let $X$, $Y$ and $Z$ denote three disjoint subsets of $V$. We say that $X$ is connected to $Y$ given $Z$ in an UG $G$ if there exist two nodes $A \in X$ and $B \in Y$ such that there exists a \textbf{single} path between $A$ and $B$ in $G$ whose nodes are all outside $X Y Z \setminus A B$. We denote such a connection statement by $con_G(X, Y | Z)$. We denote by $X \de_G Y | Z$ that an UG $G$ represents that $U_X$ is dependent on $U_Y$ given $U_Z$. We can now introduce our graphical criterion for reading dependencies from the covariance graph of a WTC graphoid.

\begin{definition}\label{def:de}
Given the covariance graph $G$ of a WTC graphoid, $X \de_G Y | Z$ if $con_G(X, Y | V \setminus X Y Z)$.
\end{definition}

The following rephrasing of the graphical criterion in the previous definition may be easier to recall: $X \de_G Y | Z$ if there exist two nodes $A \in X$ and $B \in Y$ such that there exists a \textbf{single} path between $A$ and $B$ in $G$ whose nodes are all in $A B Z$. Interestingly, the graphical criterion in the previous definition is dual to the following graphical criterion, which we developed in \citep{Pennaetal.2009} for reading dependencies from the concentration graph of a WTC graphoid.

\begin{definition}\label{def:de2}
Given the concentration graph $G$ of a WTC graphoid, $X \de_G Y | Z$ if $con_G(X, Y | Z)$.
\end{definition}

We proved in \citep[Theorems 5 and 6, Example 3]{Pennaetal.2009} that the graphical criterion in Definition \ref{def:de2} is sound and complete in certain sense. We prove in the following two theorems that the graphical criterion in Definition \ref{def:de} is also sound and complete in certain sense.

\begin{theorem}\label{the:sound}
Let $G$ be the covariance graph of a WTC graphoid $p$. If $X \de_G Y | Z$, then $X \nci_p Y | Z$ is in the WTC graphoid closure of the dependence base of $p$.
\end{theorem}

\begin{proof}
Let $X \de_G Y | Z$ hold due to a path $A:B$ with $A \in X$ and $B \in Y$. We prove the theorem by induction over the length of $A:B$. We first prove it for length one. Let $Z'$ denote the largest subset of $Z$ such that there is a path in $G$ from $A$ to every node in $Z'$ and all the nodes in these paths are in $A Z'$. Then, $B \ci_G Z'$ because, otherwise, there would be two paths between $A$ and $B$ whose nodes are all in $A B Z$, which would contradict that $X \de_G Y | Z$ holds due to $A:B$. Thus, $B \ci_p Z'$. Moreover, $A \nci_p B$ because $A$ and $B$ are adjacent in $G$. Then, $A Z' \nci_p B$ by symmetry and decomposition, which together with $B \ci_p Z'$ imply $A \nci_p B | Z'$ by symmetry and contraction2. Note that if $Z'=\emptyset$, then $A \nci_p B$ directly implies $A \nci_p B | Z'$. In any case, $A \nci_p B | Z'$ implies $A \nci_p B (Z \setminus Z') | Z'$ by decomposition. Now, note that $A Z' \ci_G Z \setminus Z'$ by definition of $Z'$ and thus $A Z' \ci_p Z \setminus Z'$, which implies $A \ci_p Z \setminus Z' | Z'$ by symmetry and weak union, which together with $A \nci_p B (Z \setminus Z') | Z'$ imply $A \nci_p B | Z$ by contraction2. Note that if $Z \setminus Z'=\emptyset$, then $A \nci_p B | Z'$ directly implies $A \nci_p B | Z$. In any case, $A \nci_p B | Z$ implies $X \nci_p Y | Z$ by symmetry and decomposition.

Assume as induction hypothesis that the theorem holds when the length of $A:B$ is smaller than $n$. We now prove it for length $n$. Let $C$ be any node in $A:B$ except $A$ and $B$. Note that $C \in Z$ and thus $A \ci_G B | Z \setminus C$, which implies $A \ci_p B | Z \setminus C$. Moreover, note that $A \de_G C | Z \setminus C$ holds due to the subpath of $A:B$ between $A$ and $C$, which we denote as $A:C$. To see it, note that $A:C$ is the only path between $A$ and $C$ in $G$ whose nodes are all in $A Z$, because if there were two such paths then there would be two paths between $A$ and $B$ in $G$ whose nodes are all in $A B Z$, which would contradict that $X \de_G Y | Z$ holds due to $A:B$. Likewise, $C \de_G B | Z \setminus C$. Moreover, $A \de_G C | Z \setminus C$ and $C \de_G B | Z \setminus C$ imply respectively $A \nci_p C | Z \setminus C$ and $C \nci_p B | Z \setminus C$ by the induction hypothesis, which together with $A \ci_p B | Z \setminus C$ imply $A \nci_p B | Z$ by weak transitivity1, which implies $X \nci_p Y | Z$ by symmetry and decomposition.

Finally, note that the above derivation of $X \nci_p Y | Z$ only makes use of the dependence base of $p$ and the nine properties introduced at the beginning of this section. Thus, $X \nci_p Y | Z$ is in the WTC graphoid closure of the dependence base of $p$.
\end{proof}

Note that we do not make use of the composition property in the proof above. However, we do use the fact that the graphical criterion in Definition \ref{def:inde2} is sound. The proof of this fact in \citep[Proposition 2.2]{BanerjeeandRichardson2003,Kauermann1996} does make use of the composition property.

\begin{theorem}\label{the:complete}
Let $G$ be the covariance graph of a WTC graphoid $p$. If $X \nci_p Y | Z$ is in the WTC graphoid closure of the dependence base of $p$, then $X \de_G Y | Z$.
\end{theorem}

\begin{proof}
Let $H$ denote the concentration graph that has the same vertices and edges as $G$. In other words, $G$ and $H$ are the same UG but with different interpretations. Note that $X \de_G Y | Z$ iff $X \de_H Y | V \setminus X Y Z$, which follows from the fact that $X \de_G Y | Z$ iff $con_G(X, Y | V \setminus X Y Z)$ iff $con_H(X, Y | V \setminus X Y Z)$ iff $X \de_H Y | V \setminus X Y Z$.

Clearly, all the dependencies in the dependence base of $p$ are identified by the graphical criterion in Definition \ref{def:de}. Therefore, it only remains to prove that this graphical criterion satisfies the nine properties introduced at the beginning of this section. We do so below with the help of $H$. Note that the graphical criterion in Definition \ref{def:de2} applied to $H$ satisfies the nine properties introduced at the beginning of this section \citep[Theorem 6, Example 3]{Pennaetal.2009}.

\begin{itemize}
\item Symmetry $Y \de_G X | Z \Rightarrow X \de_G Y | Z$.

Trivial.

\item Decomposition $X \de_G Y | Z \Rightarrow X \de_G Y W | Z$.

$X \de_G Y | Z$ implies $X \de_H Y | V \setminus X Y Z$ by definition, which implies $X \de_H Y W | V \setminus X Y Z W$ by weak union, which implies $X \de_G Y W | Z$ by definition.

\item Weak union $X \de_G Y | Z W \Rightarrow X \de_G Y W | Z$.

$X \de_G Y | Z W$ implies $X \de_H Y | V \setminus X Y Z W$ by definition, which implies $X \de_H Y W | V \setminus X Y Z W$ by decomposition, which implies $X \de_G Y W | Z$ by definition.

\item Contraction1 $X \de_G Y W | Z \land X \ci_G Y | Z W \Rightarrow X \de_G W | Z$.

$X \de_G Y W | Z$ and $X \ci_G Y | Z W$ imply respectively \mbox{$X \de_H Y W | V \setminus X Y Z W$} and $X \ci_H Y | V \setminus X Y Z W$ by definition, which imply \mbox{$X \de_H W | V \setminus X Z W$} by contraction2, which implies $X \de_G W | Z$ by definition.

\item Contraction2 $X \de_G Y W | Z \land X \ci_G W | Z \Rightarrow X \de_G Y | Z W$.

$X \de_G Y W | Z$ and $X \ci_G W | Z$ imply respectively $X \de_H Y W | V \setminus X Y Z W$ and $X \ci_H W | V \setminus X Z W$ by definition, which imply $X \de_H Y | V \setminus X Y Z W$ by contraction1, which implies $X \de_G Y | Z W$ by definition.

\item Intersection $X \de_G Y W | Z \land X \ci_G Y | Z W \Rightarrow X \de_G W | Z Y$.

$X \de_G Y W | Z$ and $X \ci_G Y | Z W$ imply respectively \mbox{$X \de_H Y W | V \setminus X Y Z W$} and $X \ci_H Y | V \setminus X Y Z W$ by definition, which imply \mbox{$X \de_H W | V \setminus X Y Z W$} by composition, which implies $X \de_G W | Z Y$ by definition.

\item Weak transitivity1 $X \de_G K | Z \land K \de_G Y | Z \land X \ci_G Y | Z \Rightarrow X \de_G Y | Z K$.

$X \de_G K | Z$, $K \de_G Y | Z$ and $X \ci_G Y | Z$ imply respectively \mbox{$X \de_H K | V \setminus X Z K$}, $K \de_H Y | V \setminus Y Z K$ and $X \ci_H Y | V \setminus X Y Z$. Moreover, \mbox{$X \de_H K | V \setminus X Z K$} and $K \de_H Y | V \setminus Y Z K$ imply respectively \mbox{$X \de_H Y K | V \setminus X Y Z K$} and $X K \de_H Y | V \setminus X Y Z K$ by symmetry and weak union, which together with $X \ci_H Y | V \setminus X Y Z$ imply respectively \mbox{$X \de_H K | V \setminus X Y Z K$} and $K \de_H Y | V \setminus X Y Z K$ by symmetry and contraction1, which together with $X \ci_H Y | V \setminus X Y Z$ imply $X \de_G Y | V \setminus X Y Z K$ by weak transitivity2, which implies $X \de_G Y | Z K$ by definition.

\item Weak transitivity2 $X \de_G K | Z \land K \de_G Y | Z \land X \ci_G Y | Z K \Rightarrow X \de_G Y | Z$.

Trivial because the antecedent involves a contradiction. To see it, note that $X \de_G K | Z$, $K \de_G Y | Z$ and $X \ci_G Y | Z K$ imply respectively \mbox{$X \de_H K | V \setminus X Z K$}, $K \de_H Y | V \setminus Y Z K$ and $X \ci_H Y | V \setminus X Y Z K$. Moreover, $X \de_H K | V \setminus X Z K$ and $K \de_H Y | V \setminus Y Z K$ imply respectively \mbox{$X \de_H Y K | V \setminus X Y Z K$} and $X K \de_H Y | V \setminus X Y Z K$ by symmetry and weak union, which together with $X \ci_H Y | V \setminus X Y Z K$ imply respectively $X \de_H K | V \setminus X Y Z K$ and $K \de_H Y | V \setminus X Y Z K$ by symmetry and composition, which together with $X \ci_H Y | V \setminus X Y Z K$ imply a contradiction as shown in \citep[Theorem 6]{Pennaetal.2009}.

\item Composition $X \de_G Y W | Z \land X \ci_G Y | Z \Rightarrow X \de_G W | Z$.

$X \de_G Y W | Z$ and $X \ci_G Y | Z$ imply respectively $X \de_H Y W | V \setminus X Y Z W$ and $X \ci_H Y | V \setminus X Y Z$ by definition, which imply $X \de_H W | V \setminus X Z W$ by intersection, which implies $X \de_G W | Z$ by definition.
\end{itemize}
\end{proof}

While Theorem \ref{the:sound} was somewhat expected because if there is a single path between $A$ and $B$ in $G$ whose nodes are all in $A B Z$ then there is no possibility of path cancellation, the combination of Theorems \ref{the:sound} and \ref{the:complete} is rather exciting: We now have a simple graphical criterion to decide whether a given dependence is or is not in the WTC graphoid closure of the dependence base of $p$, i.e. we do not need to try to find a derivation of it, which is usually a tedious task.

We devote the rest of this section to some observations that follow from the previous two theorems. A sensible question to ask is whether the graphical criterion in Definition \ref{def:de} is complete in the sense of being able to identify all the dependencies shared by all the WTC graphoids whose covariance graph is a given UG. The answer is no. An illustrative example follows.

\begin{example}\label{exa:example2}
Let $G$ denote the UG in Table \ref{tab:example}. Consider any WTC graphoid $p$ whose covariance graph is $G$. Such WTC graphoids exist by Theorems \ref{the:fcvg} and \ref{the:fcvg2}. Then, $A \nci_p B | C$ or $A \nci_p C | B$ because otherwise $A \ci_p B C$ by intersection, which is a contradiction because $A \de_G B C$ implies $A \nci_p B C$ by Theorem \ref{the:sound}. Assume $A \nci_p B | C$. Note that $B \de_G D | C$ implies $B \nci_p D | C$ by Theorem \ref{the:sound}. Then, $A \nci_p B | C$ and $B \nci_p D | C$ together with $A \ci_p D | C$, which follows from $A \ci_G D | C$, imply $A \nci_p D | B C$ by weak transitivity1. Likewise, $A \nci_p E | B C$ when assuming $A \nci_p C | B$. Then, $A \nci_p D | B C$ or $A \nci_p E | B C$, which imply $A \nci_p D E | B C$ by decomposition. However, $A \de_G D E | B C$ does not hold.
\end{example}

Note that the fact that the graphical criterion Definition \ref{def:de} is not complete in the latter sense implies that it is neither complete in the more stringent sense of being able to identify all the dependencies in the WTC graphoid at hand. Actually, no sound graphical criterion for reading dependencies from the covariance graph of a WTC graphoid can be complete in this more stringent sense. To see it, consider again Example \ref{exa:example}. Let us assume that we are dealing with $p'$. Then, no sound graphical criterion entails $A \nci_{p'} C | B$ from $G'$ because this dependence does not hold in $p$, and it is impossible to know whether we are dealing with $p$ or $p'$ on the sole basis of $G'$.

It is worth mentioning that the graphical criteria in Definitions \ref{def:de} and \ref{def:de2} complement each other, as each of them can read dependencies than the other cannot. To see it, consider the WTC graphoid $p$ in Example \ref{exa:example}. Then, $A \de_G B$ and thus $A \nci_p B$ by Theorem \ref{the:sound}. However, this dependence cannot be derived from $H$ because $A \de_H B$ does not hold. On the other hand, $A \de_H B | C D$ and thus $A \nci_p B | C D$ by \citep[Theorem 5]{Pennaetal.2009}. However, this dependence cannot be derived from $G$ because $A \de_G B | C D$ does not hold.

Again, a sensible question to ask is whether the joint use of the graphical criteria in Definitions \ref{def:de} and \ref{def:de2} is complete in the sense of being able to identify all the dependencies shared by all the WTC graphoids whose covariance and concentration graphs are two given UGs. The answer is no. An illustrative example follows.

\begin{example}\label{exa:example3}
Let $G$ denote the UG in Table \ref{tab:example}. Let $H$ denote the complete graph over $\{A, B, C, D, E, F\}$. Consider any WTC graphoid $p$ whose covariance and concentration graphs are $G$ and $H$, respectively. Such WTC graphoids exist. To see it, it suffices to take any WTC graphoid that is faithful to $G$, which exists by Theorems \ref{the:fcvg} and \ref{the:fcvg2}. Recall that we have proven in Example \ref{exa:example2} that $A \ci_p D E | B C$. However, neither $A \de_G D E | B C$ nor $A \de_H D E | B C$ holds.
\end{example}

Note that the fact that the joint use of the graphical criteria in Definitions \ref{def:de} and \ref{def:de2} is not complete in the latter sense implies that it is neither complete in the more stringent sense of being able to identify all the dependencies in the WTC graphoid at hand. Actually, no pair of sound graphical criteria for reading dependencies from the covariance and concentration graphs of a WTC graphoid can be complete in this more stringent sense. To see it, consider again Example \ref{exa:example}. Let us assume that we are dealing with $p'$. Then, no pair of sound graphical criteria entails $A \nci_{p'} C | B$ from $G'$ and $H'$ because this dependence does not hold in $p$, and it is impossible to know whether we are dealing with $p$ or $p'$ on the sole basis of $G'$ and $H'$.

The following corollary extends to WTC graphoids a result originally proven in \citep[Theorem 3]{MaloucheandRajaratnam2009} for Gaussian probability distributions. The extension is straightforward thanks to the graphical criterion in Definition \ref{def:de}.

\begin{corollary}
Let $G$ be the covariance graph of a WTC graphoid $p$. If $G$ is a forest, then $p$ is faithful to $G$.
\end{corollary}

\begin{proof}
Assume to the contrary that $p$ is not faithful to $G$. Since $G$ is the covariance graph of $p$, the previous assumption is equivalent to assume that there exist three disjoint subsets of $V$, here denoted $X$, $Y$ and $Z$, such that $X \nci_G Y | Z$ but $X \ci_p Y | Z$. However, $X \nci_G Y | Z$ implies that there must exist a path in $G$ between some node $A \in X$ and some node $B \in Y$ whose nodes are all in $A B Z$. Furthermore, since $G$ is a forest, that must be the only such path between $A$ and $B$ in $G$. However, this implies $X \de_G Y | Z$ and thus $X \nci_p Y | Z$ by Theorem \ref{the:sound}, which is a contradiction.
\end{proof}

Given the covariance graph $G$ of a WTC graphoid $p$, $X \nci_G Y | Z$ does not imply $X \nci_p Y | Z$. This is actually the reason of being of this paper. However, if $G$ is a forest, then the previous corollary proves that $X \nci_G Y | Z$ does imply $X \nci_p Y | Z$ and, moreover, that this way of reading dependencies from $G$ is complete in the strictest sense discussed above. The following corollary extends to WTC graphoids a result originally proven in \citep[Lemma 5]{MaloucheandRajaratnam2009} for Gaussian probability distributions. The extension is straightforward thanks to the graphical criteria in Definitions \ref{def:de} and \ref{def:de2}.

\begin{corollary}
Let $G$ and $H$ be, respectively, the covariance and concentration graphs of a WTC graphoid $p$. Then, $G$ and $H$ have the same connected components. Moreover, if a connected component in $G$ (respectively $H$) is a tree then the corresponding connected component in $H$ (respectively $G$) is the complete graph.
\end{corollary}

\begin{proof}
First, we prove that $G$ and $H$ have the same connected components. If two nodes $A$ and $B$ belong to the same connected component in $G$, then $A \de_G B | Z$ for some $Z \subseteq V \setminus A B$ and thus $A \nci_p B | Z$ by Theorem \ref{the:sound}. However, if $A$ and $B$ belong to different connected components in $H$, then $A \ci_H B | Z$ and thus $A \ci_p B | Z$, which is a contradiction. On the other hand, if two nodes $A$ and $B$ belong to the same connected component in $H$, then $A \de_H B | Z$ for some $Z \subseteq V \setminus A B$ and thus $A \nci_p B | Z$ by \citep[Theorem 5]{Pennaetal.2009}. However, if $A$ and $B$ belong to different connected components in $G$, then $A \ci_G B | Z$ and thus $A \ci_p B | Z$, which is a contradiction.

Now, take any connected component $\mathcal{C}$ in $G$ that is a tree. We prove that the corresponding connected component $\mathcal{D}$ in $H$ is the complete graph. If two nodes $A$ and $B$ belong to $\mathcal{C}$, then $A \de_G B | V \setminus A B$ and thus $A \nci_p B | V \setminus A B$ by Theorem \ref{the:sound}, which implies that $A$ and $B$ are adjacent in $\mathcal{D}$.

Finally, take any connected component $\mathcal{D}$ in $H$ that is a tree. We prove that the corresponding connected component $\mathcal{C}$ in $G$ is the complete graph. If two nodes $A$ and $B$ belong to $\mathcal{D}$, then $A \de_H B$ and thus $A \nci_p B$ by \citep[Theorem 5]{Pennaetal.2009}, which implies that $A$ and $B$ are adjacent in $\mathcal{C}$.
\end{proof}

Note that the opposite of the second statement in the previous corollary is not true. The following example illustrates this.

\begin{example}
Let $G$ (respectively $H$) be the covariance (respectively concentration) graph of a WTC graphoid $p$. Assume that $G$ (respectively $H$) is the complete graph and $p$ is faithful to it. Such WTC graphoids exist by Theorems \ref{the:fcvg} and \ref{the:fcvg2} (respectively Theorems \ref{the:fcg} and \ref{the:fcg2}). Then, $p$ has no independencies. Consequently, $H$ (respectively $G$) is the complete graph and $p$ is faithful to it.
\end{example}

\section{Discussion}\label{sec:discussion}

In this paper, we have provided new insight into covariance graphs by introducing a graphical criterion for reading dependencies from the covariance graph of a WTC graphoid. We have shown that WTC graphoids are not a too restrictive family of probability distributions by showing that it includes interesting discrete, Gaussian, and continuous but non-Gaussian probability distributions. We have proven that the new graphical criterion is sound and complete in certain sense. In order to prove these properties, we have had to prove first that the graphical criterion in \citep{BanerjeeandRichardson2003,Kauermann1996} for reading independencies from the covariance graph of a WTC graphoid is complete in certain sense. We have done so by proving that there are discrete, Gaussian, and continuous but non-Gaussian probability distributions that are faithful to any covariance graph. This result is also important because it implies that there exist probability distributions that covariance graphs can represent exactly but CGs cannot. Therefore, covariance graphs complement CGs. The following example illustrates this.

\begin{example}
Consider the covariance graph $G = \{A - B, B - C, C - D, D - A\}$. Consider any CG $H$ that represents the same independencies as $G$. Note that $H$ must have some edge between $A$ and $B$, $B$ and $C$, $C$ and $D$, and $D$ and $A$. However, $A$ and $C$ cannot be adjacent in $H$ because $A \ci_G C$. Likewise, $B$ and $D$ cannot be adjacent in $H$ because $B \ci_G D$. Then, $H = \{A \rightarrow B, B \leftarrow C$, $A \rightarrow D, D \leftarrow C\}$ because otherwise either $A \ci_H C | B$ or $A \ci_H C | D$ or $A \ci_H C | B D$, whereas $A \nci_G C | B$ and $A \nci_G C | D$ and $A \nci_G C | B D$. However, such an $H$ implies $B \nci_H D$ whereas $B \ci_G D$. Consequently, no CG can represent the same independencies as $G$. This implies that every probability distribution that is faithful to $G$ (recall that such probability distributions exist by Theorem \ref{the:fcg2}) is not faithful to any CG.
\end{example}

It is fair mentioning that we are not the first to note that covariance graphs complement other more popular graphical models. For instance, it follows from \citep{DrtonandRichardson2003,PearlandWermuth1994} that covariance graphs complement DAGs. Our example above simply extends this observation to CGs.

Another consequence of the faithfulness result in Theorem \ref{the:fcg2} is that it proves wrong the misconception that covariance graphs are densely connected because their edges represent marginal dependencies. Specifically, the theorem implies that there are probability distributions that are faithful to any covariance graph, no matter its topology.

Interestingly, the graphical criterion developed in this paper is dual to the one presented in \citep{Pennaetal.2009} for reading dependencies from the concentration graph of a WTC graphoid. This duality resembles the duality existing between the graphical criteria for reading independencies from concentration and covariance graphs \citep{BanerjeeandRichardson2003,Kauermann1996}. We have also shown that the new graphical criterion and the one presented in \citep{Pennaetal.2009} complement each other, as there may be dependencies that only one of them can identify.

Finally, we have pointed out some limitations of the graphical criterion introduced in this paper that suggest future lines of research. For instance, it remains an open question whether it is possible to develop a similar graphical criterion that is complete in a stricter sense than the one used in this paper. It also remains an open question whether our faithfulness result in Appendix A for regular Gaussian probability distributions can be extended to discrete probability distributions with the help of the parameterizations in \citep{Lupparellietal.2009}. Another line of action is the application of our graphical criterion in bioinformatics. In such an application, the covariance graph has to be learnt from gene expression data via, for instance, hypothesis tests. The data available for learning is typically scarce due to the high cost associated with its production. In this scenario, covariance graphs are easier to learn than Bayesian networks or concentration graphs, which are the graphical models commonly used in bioinformatics, because the former involves testing for marginal (in)dependencies whereas the latter involve testing for conditional (in)dependencies. We do not suggest with this that one should quit using Bayesian networks and concentration graphs in bioinformatics. Specifically, Bayesian networks have an important advantage over covariance graphs, namely that they can provide us with insight into the mechanistic or causal process underlying the data at hand. What we suggest is that when the data at hand is not considered enough to learn a reliable Bayesian network, one may be willing to learn a less informative but more reliable model such as a covariance graph, particularly now that we have graphical criteria for reading both dependencies and independencies off it.

Note that this paper only studies the structure of covariance graphs and, thus, it does not deal with their parameterization and/or parameter estimation. The interested reader is referred to \citep{Chaudhurietal.2007,DrtonandRichardson2003} for Gaussian models and \citep{DrtonandRichardson2008} for discrete models. However, it may be worth warning here that finding the maximum likelihood estimates of the parameters of a covariance graph can be hard, depending on the parameterization considered. This is particularly true for discrete models \citep{DrtonandRichardson2008}. The reason is that a marginal independence can imply complicated parameter constraints in some parameterizations, because a marginal dependence is a property of the joint probability distribution rather than of the relevant marginal probability distribution. An early work that showed the latter is \citep{Zentgraf1975}, where the author gives an example of two-way interactions implying three-way interaction. In summary, learning the structure of a covariance graph may be simpler than learning the structure of a concentration graph. However, estimating the parameters of the former may be harder than estimating the parameters of the latter. 

\section{Appendix A}

We strengthen the second statement in Theorem \ref{the:fcvg} by proving that, in certain measure-theoretic sense, almost all the regular Gaussian probability distributions that are Markov wrt a covariance graph are faithful to it. Although this result is not used in this paper, we consider it to be interesting in its own and, thus, we decide to report on it.

We start by recalling some results from matrix theory. See, for instance, \citep{HornandJohnson1985} for more information. Let $M=(M_{i,j})_{i, j \in V}$ denote a square matrix. Let $M_{I,J}$ with $I, J \subseteq V$ denote its submatrix $(M_{i,j})_{i \in I, j \in J}$. The determinant of $M$ can recursively be computed, for fixed $i \in V$, as $det(M)=\sum_{j \in V} (-1)^{i+j} M_{i,j} det(M_{\setminus (ij)})$, where $M_{\setminus (ij)}$ denotes the matrix produced by removing the row $i$ and column $j$ from $M$. Note then that $det(M)$ is a real polynomial in the entries of $M$. If $det(M) \neq 0$ then the inverse of $M$ can be computed as $(M^{-1})_{i,j}=(-1)^{i+j} det(M_{\setminus (ji)})/det(M)$ for all $i, j \in V$. We say that $M$ is strictly diagonally dominant if $abs(M_{i,i}) > \sum_{\{j \in V \: : \: j \neq i\}} abs(M_{i,j})$ for all $i \in V$, where $abs()$ denotes absolute value. A matrix $M$ is Hermitian if it is equal to the matrix resulting from, first, transposing $M$ and, then, replacing each entry by its complex conjugate. Clearly, a real symmetric matrix is Hermitian. A real symmetric $N \times N$ matrix $M$ is positive definite if $x^T M x >0$ for all non-zero $x \in \mathbb{R}^N$.

We continue by proving some auxiliary results. We assume hereinafter that the sample space of each random variable in $U$ is $\mathbb{R}$. Let $\mathcal{N}(G)$ denote the set of regular Gaussian probability distributions $p$ such that $A \ci_p B$ for any two nodes $A$ and $B$ that are not adjacent in $G$. Note that these are exactly the regular Gaussian probability distributions that are Markov wrt $G$ \citep[Proposition 2.2]{BanerjeeandRichardson2003,Kauermann1996}. We parameterize each probability distribution $p \in \mathcal{N}(G)$ with its mean vector $\mu$ and covariance matrix $\Sigma$. Note that the values of some of these parameters are determined by the values of other parameters or by $G$. Specifically, the following constraints apply:
\begin{itemize}
\item[C1.] $\Sigma_{i,j}=\Sigma_{j,i}$ for all $i, j$ because $\Sigma$ is symmetric.
\item[C2.] $\Sigma_{i,j}=0$ for all $i, j$ such that $i$ and $j$ are not adjacent in $G$. To see it, recall that if $i$ and $j$ are not adjacent in $G$ then $i \ci_p j$ and, thus, $\Sigma_{i,j}=0$ \citep[Corollary 2.3]{Studeny2005}.
\end{itemize}

Hereinafter, the parameters whose values are not determined by the constraints above are called non-determined (nd) parameters. However, the values the nd parameters can take are further constrained by the fact that these values must correspond to some probability distribution in $\mathcal{N}(G)$. In other words, the values the nd parameters can take are constrained by the fact that $\Sigma$ must be positive definite. That is why the set of nd parameter values satisfying this requirement are hereinafter called the nd parameter space for $\mathcal{N}(G)$. We do not work out the inequalities defining the nd parameter space because they are irrelevant for our purpose. The number of nd parameters is what we call the dimension of $G$, and we denote it as $d$. Specifically, $d=2 |V| + |G|$ where $|V|$ and $|G|$ are, respectively, the number of nodes and edges in $G$:
\begin{itemize}
\item $|V|$ due to $\mu$.
\item $|V|$ due to entries in the diagonal of $\Sigma$.
\item $|G|$ due to the entries below the diagonal of $\Sigma$ that are not identically zero. To see this, recall from the constraint C2 that there is one entry below the diagonal of $\Sigma$ that is not identically zero for each undirected edge in $G$.
\end{itemize}

\begin{lemma}\label{lem:f1}
Let $G$ be a covariance graph of dimension $d$. The nd parameter space for $\mathcal{N}(G)$ has positive Lebesgue measure wrt $\mathbb{R}^d$.
\end{lemma}

\begin{proof}
Since we do not know a closed-form expression of the nd parameter space for $\mathcal{N}(G)$, we take an indirect approach to prove the result. Recall that, by definition, the nd parameter space for $\mathcal{N}(G)$ is the set of real values such that, after the extension determined by the constraints C1 and C2, $\Sigma$ is positive definite. Therefore, the nd parameters $\mu$ can take values independently of the nd parameters in $\Sigma$. However, the nd parameters in $\Sigma$ cannot take values independently one of another because, otherwise, $\Sigma$ may not be positive definite. However, if the entries in the diagonal of $\Sigma$ take values in $(|V|-1, \infty)$ and the rest of the nd parameters in $\Sigma$ take values in $[-1,1]$, then the nd parameters in $\Sigma$ can take values independently one of another. To see it, note that in this case $\Sigma$ will always be Hermitian, strictly diagonally dominant, and with strictly positive diagonal entries, which implies that $\Sigma$ will always be positive definite \citep[Corollary 7.2.3]{HornandJohnson1985}.

The subset of the nd parameter space of $\mathcal{N}(G)$ described in the paragraph above has positive volume in $\mathbb{R}^d$ and, thus, it has positive Lebesgue measure wrt $\mathbb{R}^d$. Then, the nd parameter space of $\mathcal{N}(G)$ has positive Lebesgue measure wrt $\mathbb{R}^d$.
\end{proof}

\begin{lemma}\label{lem:polynomial}
Let $G$ be a covariance graph. For every $i, j \in V$ and $K \subseteq V \setminus ij$, there exists a real polynomial $S(i,j,K)$ in the nd parameters in the parameterization of the probability distributions in $\mathcal{N}(G)$ such that, for every $p \in \mathcal{N}(G)$, $i \ci_p j | K$ iff $S(i,j,K)$ vanishes for the nd parameter values coding $p$.
\end{lemma}

\begin{proof}
Let $\Sigma$ denote the covariance matrix of $p$. Note that $i \ci_p j | K$ iff $((\Sigma_{ijK,ijK})^{-1})_{i,j}=0$ \citep[Proposition 5.2]{Lauritzen1996}. Recall that $((\Sigma_{ijK,ijK})^{-1})_{i,j}=(-1)^{\alpha_{i,j}} det(\Sigma_{iK,jK}) / det(\Sigma_{ijK,ijK})$ with $\alpha_{i,j} \in \{0,1\}$. Moreover, note that $det(\Sigma_{ijK,ijK}) > 0$ because $\Sigma$ is positive definite \citep[p. 237]{Studeny2005}. Then, $i \ci_p j | K$ iff $det(\Sigma_{iK,jK})=0$. Moreover, as noted in Section \ref{sec:preliminaries}, $det(\Sigma_{iK,jK})$ is a real polynomial in the entries of $\Sigma$. Thus, $i \ci_p j | K$ iff a real polynomial $R(i,j,K)$ in the entries of $\Sigma$ vanishes. Recall that each entry of $\Sigma$ that is not identically zero corresponds to one of the nd parameters in the parameterization of the probability distributions in $\mathcal{N}(G)$. Therefore, $R(i,j,K)$ can be expressed as a real polynomial $S(i,j,K)$ in the nd parameters. Therefore, $i \ci_p j | K$ iff $S(i,j,K)$ vanishes for the nd parameter values coding $p$.
\end{proof}

We interpret the polynomial in the previous lemma as a real function on a real Euclidean space that includes the nd parameter space for $\mathcal{N}(G)$. We say that the polynomial in the previous lemma is non-trivial if not all the values of the nd parameters are solutions to the polynomial. This is equivalent to the requirement that the polynomial is not identically zero.

\begin{lemma}\label{lem:f2}
Let $G$ be a covariance graph of dimension $d$. The subset of the nd parameter space for $\mathcal{N}(G)$ that corresponds to the probability distributions in $\mathcal{N}(G)$ that are not faithful to $G$ has zero Lebesgue measure wrt $\mathbb{R}^d$.
\end{lemma}

\begin{proof}
Recall that $\mathcal{N}(G)$ contains exactly the regular Gaussian distributions that are Markov wrt $G$. Therefore, for any probability distribution $p \in \mathcal{N}(G)$ not to be faithful to $G$, $p$ must satisfy some independence that is not entailed by $G$. That is, there must exist three disjoint subsets of $V$, here denoted as $I$, $J$ and $K$, such that $I \nci_G J | K$ but $I \ci_p J | K$. However, if $I \nci_G J | K$ then $i \nci_G j | K$ for some $i \in I$ and $j \in J$. Furthermore, if $I \ci_p J | K$ then $i \ci_p j | K$ by symmetry and decomposition. By Lemma \ref{lem:polynomial}, there exists a real polynomial $S(i,j,K)$ in the nd parameters in the parameterization of the probability distributions in $\mathcal{N}(G)$ such that, for every $q \in \mathcal{N}(G)$, $i \ci_q j | K$ iff $S(i,j,K)$ vanishes for the nd parameter values coding $q$. Furthermore, $S(i,j,K)$ is non-trivial \citep[Theorem 3.1]{Kauermann1996}. Let $sol(i, j, K)$ denote the set of solutions to the polynomial $S(i,j,K)$. Then, $sol(i, j, K)$ has zero Lebesgue measure wrt $\mathbb{R}^d$ because it consists of the solutions to a non-trivial real polynomial in real variables (the nd parameters) \citep{Okamoto1973}. Then, $sol= \bigcup_{\{I, J, K \subseteq V \: \mbox{\scriptsize disjoint} \: : \: I \nci_G J | K\}} \bigcup_{\{i \in I, j \in J \: : \: i \nci_G j | K\}} sol(i, j, K)$ has zero Lebesgue measure wrt $\mathbb{R}^d$, because the finite union of sets of zero Lebesgue measure has zero Lebesgue measure too. Consequently, the subset of the nd parameter space for $\mathcal{N}(G)$ that corresponds to the probability distributions in $\mathcal{N}(G)$ that are not faithful to $G$ has zero Lebesgue measure wrt $\mathbb{R}^d$ because it is contained in $sol$.
\end{proof}

In summary, it follows from Lemmas \ref{lem:f1} and \ref{lem:f2} that, in the measure-theoretic described, almost all the elements of the nd parameter space for $\mathcal{N}(G)$ correspond to probability distributions in $\mathcal{N}(G)$ that are faithful to $G$. Since this correspondence is clearly one-to-one, it follows that almost all the regular Gaussian distributions in $\mathcal{N}(G)$ are faithful to $G$. We think that this result can easily be extended to strictly positive discrete probability distributions with the help of the parameterizations proposed in \citep{Lupparellietal.2009}. We do not elaborate further on this issue in this paper though.

A word of caution is due at this point. It may be tempting to infer from the measure-theoretic results above that every regular Gaussian probability distribution $p$ one encounters in reality is almost surely faithful to its covariance graph $G$. This may lead one to conclude that our graphical criterion for reading from $G$ dependencies holding in $p$ is not needed, since $X \nci_G Y | Z$ almost surely implies $X \nci_p Y | Z$. This argument is valid if $p$ has been drawn from $\mathcal{N}(G)$ at random. However, we believe that $p$ is more likely to have been carefully engineered (e.g. by natural evolution in the case of the gene networks mentioned in Section \ref{sec:intro}) than to have been drawn at random. Consequently, and despite the measure-theoretic results above, we think that one cannot safely assume that $p$ is almost surely faithful to $G$. Hence, the need of the graphical criterion proposed in this paper.

\section*{Acknowledgments}

We thank the anonymous reviewers and the editor for their thorough review of this manuscript. We are particularly grateful to the second reviewer for suggesting us to use copulas in the proofs of Theorems \ref{the:fcg2} and \ref{the:fcvg}. This work is funded by the Center for Industrial Information Technology (CENIIT) and a so-called career contract at Link\"oping University.

\end{document}